\documentclass{new_tlp}

\usepackage{framed}
\usepackage{algorithm}
\usepackage[noend]{algpseudocode}
\usepackage{amssymb}
\usepackage{pifont}
\usepackage{amsmath}%
\usepackage{amsfonts}%
\usepackage{ragged2e}
\usepackage{rotating}
\usepackage{listings}
\let\proof\relax

\usepackage{amsthm}
\usepackage[justification=centering]{caption}

\usepackage{hyperref}
\usepackage{xcolor}
\usepackage[shortlabels]{enumitem}
\usepackage[justification=centering]{caption}
\theoremstyle{definition} 
\usepackage{color,soul}
\usepackage{lscape}

\newcommand\bcmdtab{\noindent\bgroup\tabcolsep=0pt%
  \begin{tabular}{@{}p{10pc}@{}p{20pc}@{}}}
\newcommand\ecmdtab{\end{tabular}\egroup}

\newtheorem{innercustomthm}{Theorem}
\newenvironment{customthm}[1]
  {\renewcommand\theinnercustomthm{#1}\innercustomthm}
  {\endinnercustomthm}

  \title[Theory and Practice of Logic Programming]
        {Using Answer Set Programming for Commonsense Reasoning in the Winograd Schema Challenge}

  \author[Arpit Sharma]
         {ARPIT SHARMA\\
         Arizona State University, Tempe, USA\\
         \email{asharm73@asu.edu}}

\newtheorem{lemma}{Lemma}[section]
\newtheorem{definition}{Definition} % [section]
\newtheorem{example}{Example} % [section]
\newtheorem{theorem}{Theorem} % [section]

\begin{document}
% \nocite{*}% includes all entries of BibTeX database into the list of references.

% \label{firstpage}

\maketitle
\begin{abstract}
 The Winograd Schema Challenge (WSC) is a natural language understanding task proposed as an alternative to the Turing test in 2011. In this work we attempt to solve WSC problems by reasoning with additional knowledge. By using an approach built on top of graph-subgraph isomorphism encoded using Answer Set Programming (ASP) we were able to handle 240 out of 291 WSC problems. The ASP encoding allows us to add additional constraints in an elaboration tolerant manner. In the process we present a graph based representation of WSC problems as well as relevant commonsense knowledge. This paper is under consideration for acceptance in TPLP
\end{abstract}


% \tableofcontents

\section{Introduction}
The Winograd Schema Challenge (WSC) \cite{levesque2011winograd} is a natural language understanding task. It is made up of special types of pronoun resolution problems. Each WSC problem consists of a sequence of sentences (currently 1-3) which contain a definite pronoun. A WSC problem also contains a binary question about the sentences such that the answer to the question provides the most natural resolution for the concerned pronoun. Additionally, two answer choices for the question are also provided. The answer choices are always present in the sentences. The goal in the WSC challenge is to determine the correct answer choice. Following is an example WSC problem. 
\vspace{-5pt}
\begin{framed}
\vspace{-5pt}
\noindent
% \textbf{Example 1:}\\
\textbf{Sentences:} The fish ate the worm. \textbf{It$_{pronoun}$} was tasty\\\textbf{Question:} What was tasty? \textbf{Answer Choices:} a) fish b) worm
\vspace{-5pt}
\end{framed}
\vspace{-5pt}
A WSC problem also specifies an ``alternate word'' for a ``special word'' in the sentences. Replacing the ``special word'' by the ``alternate word'' changes the resolution of the pronoun. In the example above, the special word is \textit{tasty} and the alternate word is \textit{hungry}. Thus every schema represents a pair of coreference resolution problems that are almost identical but have different answers. Based on our analysis of how people solve the WSC problems, it suggests that for solving them a program would have to use relevant world knowledge. For example to solve the above question, the knowledge that \textit{`something that is eaten may be tasty'} is needed.

Earlier attempts to solve the challenge are mainly based on two different approaches. Works such as \cite{schuller2014tackling} and \citeN{bailey2015winograd} solve 8 and 72 WSC problems respectively by reasoning with the explicitly provided knowledge. Such works presented algorithms which take a WSC problem and a suitable knowledge as input and produce the solution of the problem. Other attempts however utilize the recent advancement in the field of neural language modelling. For example the language model in \citeN{radford2019language} correctly answered 193 out of 273 WSC problems by predicting the more plausible answer choice based on the support generated by a language model trained on large body of text. %\cite{trinh2018simple}

In this work we attempt to solve the WSC by reasoning with additional knowledge. We define an algorithm which is built on top of graph-subgraph isomorphism \cite{cordella2004sub}. By using an Answer Set Programming (ASP) \cite{baral2003knowledge,gelfond1988stable} based implementation of the algorithm, we were able to tackle 240 out of 291 WSC problems. The motivation behind using ASP is that we would like the process of adding new constraints to be easier. It plays an important part in the isomorphism detection step of the algorithm where the nodes in two graphs are paired based on a set of constraints. Adding new constraints in other high level languages such as python would take one to delve deep into the code to identify the actual place of injection whereas it can be easily accomplished in ASP by writing a new constraint anywhere in the code. The main contributions of this work are summarized below.
\begin{itemize}
	\item a graph based representations of WSC sentences and commonsense knowledge,
	\item Winograd Schema Challenge Reasoning (WiSCR) algorithm,% to handle the 10 out of 12 knowledge types, and
	\item an ASP implementation of the WiSCR algorithm, and
	\item an experimental evaluation of the implementation showing that it handles 240 out of 291 WSC problems. This is accomplished by performing three experiments, one of which involves an automatic approach to extract knowledge from text.
	% on the original WSC corpus and a corpus inspired by the WSC.
\end{itemize}

The rest of the paper is organized as follows. Sections 2 and 3 describe graphical representations of a WSC problem and a piece of knowledge. Section 4 details the reasoning algorithm and its ASP implementation. Section 5 presents the evaluation results of the ASP implementation. Section 6 provides the literature review. Finally Section 7 presents our conclusion.

\vspace{-10pt}
\section{Graphical Representation of a WSC Problem}
Graphical meaning representations are popular for natural languages such as English. It is because of their simplicity, readability and ability to be easily processed, that in the recent years there has been a significant amount of progress \cite{sharma2015identifying,banarescu2013abstract} in defining graphical representations for natural language and development of systems which can automatically parse a natural language text into those representations. Inspired by such representations, in this work we use a graphical schema to represent the sentences in a WSC problem, and a piece of knowledge.

In the following section, we define a graphical representation of a sequence of English sentences in a WSC problem. For that reason we define a set of tokens in a sequence of sentences, a POS (part-of-speech) tagging function which maps each token in a sequence of sentences to a POS tag, a class mapping function which maps each token in a sequence of sentences to its class (or type) and finally we define a graphical representation of a sequence of sentences by using a POS tagging function and a class mapping function. The nodes in the graphical representation are made up of the tokens in the sentences and the classes of the tokens. The edge labels in the graphical representation are from a set of binary relations between two nodes in the representation. 

\begin{definition}[\textbf{Set of Tokens in a Sequence of Sentences}]\label{def:tokens}
    Let $\mathcal{S}$ = ($\mathcal{S}_1$, $\mathcal{S}_2$, ..., $\mathcal{S}_n$), $n\geq 1$, be a sequence of English sentences, $\mathcal{W}_i$ be the sequence of words in the sentence $\mathcal{S}_i$ 
    %in $\mathcal{S}$ 
    and $\mathcal{W}_\mathcal{S}=\mathcal{W}_1 ^\frown \mathcal{W}_2 ^\frown ... ^\frown \mathcal{W}_n$ be the concatenation of the word sequences. Then the set of tokens $\mathbb{T}(\mathcal{S})$ is defined as follows:
    \begin{center}
        $\mathbb{T}(\mathcal{S})$ = \{$w\_i$ $\vert$ $w$ is the $i$th word in $\mathcal{W}_\mathcal{S}$\}
    \end{center}
\end{definition}

\begin{example}
Let us consider the sequence of English sentences $\mathcal{S}$ = (\textit{`The man could not lift his son because he was so weak.'}) where $\mathcal{S}$ contains only one sentence. Then,% the set of tokens in the sequence is, 
\begin{center}
$\mathbb{T}(\mathcal{S})$ = $\{The\_1$, $man\_2$, $could\_3$, $not\_4$, $lift\_5$, $his\_6$, $son\_7$, $because\_8$, $he\_9$, $was\_10$, $so\_11$, $weak\_12\}$.
\end{center}
\end{example}

\begin{definition}[\textbf{A POS Tagging Function}]\label{def:pos}
    Let $\mathcal{S}$ be a sequence of one or more English sentences, $\mathbb{T}(\mathcal{S})$ be the set of tokens in  $\mathcal{S}$. Then, the POS (Part-Of-Speech) tagging function $f_{\mathcal{S}}^{pos}$ maps an element in $\mathbb{T}(\mathcal{S})$ to an element in the set $\{verb$, $noun$, $pronoun$, $adverb$, $adjective$, $other\}$, i.e., 
    %Mathematically, 
    \begin{center}
        $f_{\mathcal{S}}^{pos}$ : $\mathbb{T}(\mathcal{S})$ $\rightarrow$ $\{verb, noun, pronoun, adverb, adjective, other\}$
    \end{center}
\end{definition}

\begin{example}\label{ex:pos}
Let us consider the sequence of English sentences \textit{`The man could not lift his son because he was so weak.'} The set of tokens in the sequence is as shown in the Example 1. %, $\mathbb{T}(\mathcal{S})$ = $\{The\_1$, $man\_2$, $could\_3$, $not\_4$, $lift\_5$, $his\_6$, $son\_7$, $because\_8$, $he\_9$, $was\_10$, $so\_11$, $weak\_12\}$. 
Then an example of a mapping produced by a POS tagging function is,

\begin{lstlisting}
$f_{\mathcal{S}}^{pos}(The\_1)$ = $other$
$f_{\mathcal{S}}^{pos}(man\_2)$ = $noun$ 
$f_{\mathcal{S}}^{pos}(could\_3)$ = $verb$
$f_{\mathcal{S}}^{pos}(not\_4)$ = $adverb$ 
$f_{\mathcal{S}}^{pos}(lift\_5)$ = $verb$
$f_{\mathcal{S}}^{pos}(his\_6)$ = $pronoun$ 
$f_{\mathcal{S}}^{pos}(son\_7)$ = $noun$
$f_{\mathcal{S}}^{pos}(because\_8)$ = $other$
$f_{\mathcal{S}}^{pos}(he\_9)$ = $pronoun$
$f_{\mathcal{S}}^{pos}(was\_10)$ = $verb$
$f_{\mathcal{S}}^{pos}(so\_11)$ = $adverb$
$f_{\mathcal{S}}^{pos}(weak\_12)$ = $adjective$
\end{lstlisting}

\end{example}

\begin{definition}[\textbf{A Class Mapping Function}]\label{def:class}
     Let $\mathcal{S}$ be a sequence of one or more English sentences, $\mathbb{T}(\mathcal{S})$ be the set of tokens in  $\mathcal{S}$. Then, the class mapping function $f_{\mathcal{S}}^{class}$ maps an element of $\mathbb{T}(\mathcal{S})$ to an element in a set $\mathbb{C}$, i.e., $f_{\mathcal{S}}^{class}:\mathbb{T}(\mathcal{S})\rightarrow \mathbb{C}$ where the set $\mathbb{C}$ is a union of three sets $\mathbb{C}_1$, $\mathbb{C}_2$ and \{$\phi$\} such that, 
    
    \begin{itemize}[wide, nosep, labelindent = 2pt, topsep = 0.5ex]
        \item $\mathbb{C}_1$  = \{$c$ $\vert$ $c$ is the lemmatized\footnote{\url{https://nlp.stanford.edu/IR-book/html/htmledition/stemming-and-lemmatization-1.html}, \url{https://www.thoughtco.com/what-is-base-word-forms-1689161}} form of $w$ where $w\_i\in \mathbb{T}(\mathcal{S})$ and $f_{\mathcal{S}}^{pos}(w\_i)$ $\in$ $\{verb$, $adverb$, $adjective\}$\}
        
        \item $\mathbb{C}_2$  = $\{object$, $person$, $group$, $location$, $quantity$, $shape$, $animal$, $plant$, $cognition$, $communication$, $event$, $feeling$, $act$, $motive$, $phenomenon$, $possession$, $process$, $relation$, $state$, $time\}$\footnote{Inspired from WordNet \cite{miller1995wordnet} lexicographer files \url{https://wordnet.princeton.edu/documentation/lexnames5wn}}
    \end{itemize}

    and,
    \[
      f_{\mathcal{S}}^{class}(x) =
      \begin{cases}
         c_1\in \mathbb{C}_1 & \text{if $f_{\mathcal{S}}^{pos}(x)\in\{verb, adjective, adverb\}$}\\
         c_2\in \mathbb{C}_2  & \text{if $f_{\mathcal{S}}^{pos}(x)\in\{noun, pronoun\}$}\\
        \phi        & \text{otherwise}
      \end{cases}
    \]
\end{definition}

\begin{example}\label{ex:class}
Let us consider the sequence of English sentences \textit{`The man could not lift his son because he was so weak.'} The set of tokens in the sequence is as shown in the Example 1. Also let a mapping produced by a POS tagging function is as shown in the Example \ref{ex:pos} above. Then an example of a mapping produced by a class mapping function is, % \noindent
\begin{lstlisting}
$f_{\mathcal{S}}^{class}(The\_1) = \phi$
$f_{\mathcal{S}}^{class}(man\_2) = person$
$f_{\mathcal{S}}^{class}(could\_3) = can$
$f_{\mathcal{S}}^{class}(not\_4) = not$
$f_{\mathcal{S}}^{class}(lift\_5) = lift$
$f_{\mathcal{S}}^{class}(his\_6) = person$
$f_{\mathcal{S}}^{class}(son\_7) = person$
$f_{\mathcal{S}}^{class}(because\_8) = \phi$
$f_{\mathcal{S}}^{class}(he\_9) = person$
$f_{\mathcal{S}}^{class}(was\_10) = be$
$f_{\mathcal{S}}^{class}(so\_11) = so$
$f_{\mathcal{S}}^{class}(weak\_12) = weak$
\end{lstlisting}
\end{example}

\begin{definition}[\textbf{A Formal Representation of a Sequence of One or More English Sentences}]\label{def:text} 
Let $\mathcal{S}$ be a sequence of English sentences, $\mathbb{T}(\mathcal{S})$ be a set of tokens in $\mathcal{S}$, $f_\mathcal{S}^{pos}$ be a POS tagging function and $f_\mathcal{S}^{class}$ be a class mapping function. Then, a formal representation of $\mathcal{S}$ is an edge labeled directed acyclic graph, $\mathcal{G}_\mathcal{S} = (\mathbb{V},\mathbb{E},f)$. The set of vertices $\mathbb{V}$, is a union of two disjoint sets $\mathbb{V}_1$ and $\mathbb{V}_2$, such that,
	
	\begin{itemize}[wide, nosep, labelindent = 0pt, topsep = 0ex]
		\item $\mathbb{V}_1$ = \{$w\_i$ $\vert$ $w\_i\in \mathbb{T}(\mathcal{S})$ and $f_\mathcal{S}^{pos}(w\_i)\in$\{$verb$, $adverb$, $adjective$, $noun$, $pronoun$\}\}
		
		\item $\mathbb{V}_2$ = \{$c$ $\vert$ $f_\mathcal{S}^{class}(w\_i)=c$ where $w\_i\in\mathbb{V}_1$\}
		
	\end{itemize}
    The nodes in $\mathbb{V}_1$ are called $instance$ nodes and the nodes in $\mathbb{V}_2$ are called $class$ nodes.
    
\noindent
    $\mathbb{E}\subseteq\mathbb{V}\times\mathbb{V}$, has following properties,
	\begin{itemize}[wide, nosep, labelindent = 0pt, topsep = 0ex]
	\item $\mathbb{E}$ is a union of the two disjoint sets $\mathbb{E}_1$ and $\mathbb{E}_2$,
	\item $(v_1,v_2)\in \mathbb{E}_1$ if $v_1\in \mathbb{V}_1$ and $v_2\in \mathbb{V}_1$, where $(v_1,v_2)$ represents a directed edge between the nodes $v_1$ and $v_2$,
	\item $(v_1,v_2)\in \mathbb{E}_2$ if $v_1\in \mathbb{V}_1$ and $v_2\in \mathbb{V}_2$, where $(v_1,v_2)$ represents a directed edge between the nodes $v_1$ and $v_2$,
	\item if $(v_1,v_2)\in \mathbb{E}_2$ then there does not exist $v\in\mathbb{V}_2$ such that $(v_1,v)\in \mathbb{E}_2$ where $v\neq v_2$. This means that $v_1$ has only one class node as its successor. This is because a concept can be of one type only  in this representation.
	\end{itemize}
\noindent
$f:\mathbb{E}\rightarrow\mathbb{L}\cup\{instance\_of\}$, is an edge labelling function where $\mathbb{L}$ is a set of binary relations between two nodes in $\mathbb{V}_1$ and \textit{instance\_of} is a binary relation between a node in $\mathbb{V}_1$ and a node in $\mathbb{V}_2$, i.e.
	\[
      f((v_1,v_2)) =
      \begin{cases}
        l\in \mathbb{L} & \text{if $(v_1,v_2)\in \mathbb{E}_1$}\\
        ``instance\_of" & \text{if $(v_1,v_2)\in \mathbb{E}_2$}\\%, where $x_a$ is the 
      \end{cases}
    \]
\end{definition}

\begin{example}
Let us consider the sequence of sentences \textit{`The man could not lift his son because he was so weak.'}, POS mapping shown in Example \ref{ex:pos} and class mapping shown in Example \ref{ex:class}. Then a representation of the sentences is shown in Figure \ref{fig:sent1}. All the edges labels other than \textit{instance\_of} part of a predefined set of binary relations between two nodes (as mentioned in the Definition \ref{def:text}). In this work, these relations are from the relations in a semantic parser called K-Parser \cite{sharma2015identifying}. 
\begin{figure}[ht]
	\centering
	\includegraphics[width=\textwidth]{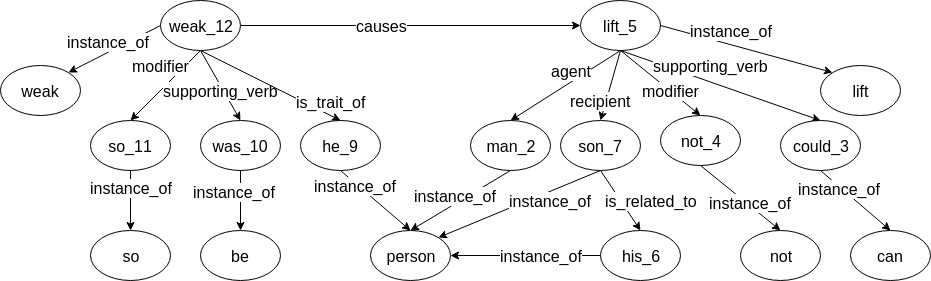}
	\caption{A Graphical Representation of Sequence of Sentences in a WSC Problem, \textit{``The man could not lift his son because he was so weak.''}}
	\label{fig:sent1}
	\vspace{-15pt}
\end{figure} 

\end{example}

\section{Graphical Representation of a Piece of Knowledge}
The WSC corpus was created in a way that each problem in it requires an additional knowledge. Let us consider the following WSC example.% to understand an identified knowledge. 
\vspace{-5pt}
\begin{framed}
\noindent
% \textbf{Example 3:}\\
\textbf{Sentence:} The man could not lift his son because \textbf{he$_{pronoun}$} was so weak.
\\\textbf{Question:} Who was weak? \textbf{Answer Choices:} a) man b) son
\end{framed}
\vspace{-5pt}

\noindent
The above problem can be correctly solved by using the commonsense knowledge that,
% \begin{center}
\textit{``someone being weak prevents her to lift someone else''}. 
% \textit{x is weak \textbf{may prevent} x lifts something}
% \end{center}
% \noindent
This knowledge can be written as, %in the following format to allow co-reference resolution,
% In other words,
% \vspace{-15pt}
% \begin{center}
    % \textbf
    % {if} 
    \textit{``\textbf{if} person1 can not lift someone because person2 is weak %\\\textbf
    \textbf{then} person1 
    % \textbf
    \textbf{is same as} person2''}. 
% \end{center}
% \vspace{-15pt}
% \noindent
Intuitively, it means that if \textit{person2} being weak prevents \textit{person1} from lifting something then \textit{person1} is same as \textit{person2}. Such a knowledge is made up of two parts. The first part is an \textit{if-condition}, which consists of an English sentence, i.e., \textit{`person1 can not lift someone because person2 is weak'}. The second part of the knowledge is the consequent of the \textit{if-condition}. The consequent is always an \textit{`is same as'} commutative relationship between two words (e.g., $person1$ and $person2$ above) in the sentence. Such a knowledge and its graphical representation are formally defined below.
%Following are the formal definitions of a knowledge and a representation of a knowledge.

\begin{definition}[\textbf{A Piece of Knowledge}]\label{def:know}
A piece of knowledge $\mathcal{K}$ is a statement of the form `\textbf{IF} $\mathcal{S}$ \textbf{THEN} $x$ \textbf{is same as} $y$' where $\mathcal{S}$ is an English sentence, $\mathbb{T}(\mathcal{S})$ is a set of tokens in $\mathcal{S}$, $x,y \in \mathbb{T}(S)$, $f_{\mathcal{S}}^{pos}(x)=noun$ and $f_{\mathcal{S}}^{pos}(y)=noun$, where $f_{\mathcal{S}}^{pos}$ is a POS tagging function.
\end{definition}

\begin{example}
An example of a piece of knowledge is,
% \begin{center}
    \textbf{IF} \textit{`person1 can not lift someone because person2 is weak'} \textbf{THEN} \textit{person1\_1} \textbf{is same as} \textit{person2\_7}. 
% \end{center}
\end{example}

\begin{definition}[\textbf{A Graphical Representation of a Piece of Knowledge}]\label{def:know_rep}
Let $\mathcal{K}$ = `\textbf{IF} $\mathcal{S}$ \textbf{THEN} $x$ \textbf{is same as} $y$' be a piece of knowledge %be a knowledge %that belongs to either one of the types 1 through 10 (as shown in the previous section) 
where $\mathcal{S}$ is an English sentence, $x$ and $y$ are tokens in $\mathcal{S}$ and % and $x$, $y$ are two tokens in $\mathcal{S}$. 
$\mathcal{G}_\mathcal{S} = (\mathbb{V}_\mathcal{S},\mathbb{E}_\mathcal{S},f_\mathcal{S})$ be a graphical representation of $\mathcal{S}$. % such that $x$ and $y$ are two nodes in $\mathcal{G}_\mathcal{S}$. 
Then, a graphical representation of $\mathcal{K}$ is an edge labeled directed graph $\mathcal{G_K} = (\mathbb{V}_\mathcal{K},\mathbb{E}_\mathcal{K},f_\mathcal{K})$, such that,

\begin{itemize}[wide, nosep, labelindent = 0pt, topsep = 0ex]
    \item $\mathbb{V}_\mathcal{K}$ = $\mathbb{V}_\mathcal{S}$,
    
    % \item $L_K$ = $L_t \bigcup \{``is\_same\_as"\}$ where $f_t:E_t \rightarrow L_t$,
    
    \item $\mathbb{E}_\mathcal{K}$ = $\mathbb{E}_\mathcal{S} \bigcup \{(x,y), (y,x)\}$, and
    
    \item 
    \[
      f_\mathcal{K}((v_1,v_2)) =
      \begin{cases}
        f_\mathcal{S}((v_1,v_2)) & \text{if $(v_1,v_2)\in \mathbb{E}_\mathcal{S}$} \\
        ``is\_same\_as" & \text{Otherwise}\\
      \end{cases}
    \]
   
\end{itemize}
Here, we say that $f_\mathcal{K}$ is defined using $f_\mathcal{S}$.
\end{definition}

\begin{example}
An example of a representation of a piece of knowledge is  shown in Figure \ref{fig:ch3_know_ex}.

\begin{figure}[ht]
  \centering
%    \reflectbox{%
      \includegraphics[width=\textwidth]{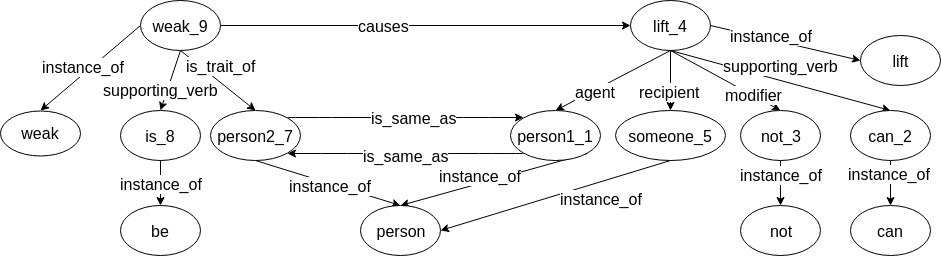}
  \caption{Graphical Representation of the Knowledge, ``\textbf{IF} \textit{person1 can not lift someone because person2 is weak} \textbf{THEN} \textit{person1\_1 is same as person2\_7}''}
  \label{fig:ch3_know_ex}
  \vspace{-12pt}
\end{figure}

\end{example}

\vspace{-10pt}
\section{Reasoning with Commonsense Knowledge}
\vspace{-5pt}
In this work we defined a reasoning algorithm for solving the WSC problems. 
The algorithm takes graphical representations of a WSC problem and a piece of knowledge as input and outputs the answer of the WSC problem if it is inferred from the inputs. 
As per the problem definition the correct answer provides the `most natural resolution' for the pronoun in the WSC sentences. In the following two definitions we formally defined the `most natural resolution' and the answer of a WSC problem with respect to the graphical representations of a WSC problem and a piece of knowledge needed to answer it.

\begin{definition}[\textbf{Most Natural Resolution}]\label{def:nat_res}
Let $\mathcal{S}$ be a sequence of sentences in a WSC problem, $\mathcal{G_S}$ = ($\mathbb{V}_\mathcal{S}$,$\mathbb{E}_\mathcal{S}$,$f_\mathcal{S}$) be a graphical representation of $\mathcal{S}$, $\mathcal{G_S'}$ = ($\mathbb{V}_\mathcal{S}'$,$\mathbb{E}_\mathcal{S}'$,$f_\mathcal{S}'$) be a subgraph of $\mathcal{G_S}$ such that $\mathbb{V}_\mathcal{S}'=\mathbb{V}_\mathcal{S}-\mathbb{V}_\mathcal{S}^c$ where $\mathbb{V}_\mathcal{S}^c$ is the set of all the class nodes in $\mathcal{G_S}$, $f_\mathcal{S}'=f_\mathcal{S}$ and $\mathbb{E}_\mathcal{S}'=\mathbb{E}_\mathcal{S}-\mathbb{E}_\mathcal{S}^c$ where $e\in \mathbb{E}_\mathcal{S}^c$ iff $f_\mathcal{S}(e)=``instance\_of"$. Let $\mathcal{G_K}$ = ($\mathbb{V}_\mathcal{K}$,$\mathbb{E}_\mathcal{K}$,$f_\mathcal{K}$) be a graphical representation of a piece of knowledge where $f_\mathcal{K}$ is defined using $f_\mathcal{S}$, $\mathcal{G_K'}$ = ($\mathbb{V}_\mathcal{K}'$,$\mathbb{E}_\mathcal{K}'$,$f_\mathcal{K}'$) be a subgraph of $\mathcal{G_K}$ such that $\mathbb{V}_\mathcal{K}'=\mathbb{V}_\mathcal{K}-\mathbb{V}_\mathcal{K}^c$ where $\mathbb{V}_\mathcal{K}^c$ is the set of all the class nodes in $\mathcal{G_K}$, $f_\mathcal{K}'=f_\mathcal{K}$ and $\mathbb{E}_\mathcal{K}' = \mathbb{E}_\mathcal{K}-\mathbb{E}_\mathcal{K}^c$ where $e\in \mathbb{E}_\mathcal{K}^c$ iff $f_\mathcal{K}(e) \in \{is\_same\_as,instance\_of\}$. Also, let $\mathbb{M}$ be a set of pairs of the form ($a$,$b$) such that either all of the below conditions are satisfied or $\mathbb{M}=\emptyset$.
\begin{itemize}[wide, nosep, labelindent = 0pt, topsep = 0ex]
    \item $a\in \mathbb{V}_\mathcal{S}'$ and $b\in \mathbb{V}_\mathcal{K}'$, 
    
    \item $a$ and $b$ are instances of same class, i.e., $(a,i)\in \mathbb{E}_\mathcal{S}$, $(b,i)\in \mathbb{E}_\mathcal{K}$, $f_\mathcal{S}((a,i))=instance\_of$ and $f_\mathcal{K}((b,i))=instance\_of$
    
    \item if for every pair ($a$,$b$)$\in$ $\mathbb{M}$, $a$ is replaced by $b$ in $\mathbb{V}_\mathcal{S}'$ then $\mathcal{G_K'}$ becomes a  
    subgraph of the node replaced $\mathcal{G_S'}$
\end{itemize}

% there exist two nodes, 
\noindent
Then we say that $x\in\mathbb{V}_\mathcal{S}'$ provides the \textbf{`most natural resolution'} for $y\in \mathbb{V}_\mathcal{S}'$ if ($x$,$n_1$)$\in \mathbb{M}$, ($y$,$n_2$)$\in \mathbb{M}$ and either one of the following is true
\begin{itemize}[wide, nosep, labelindent = 0pt, topsep = 0ex]
    \item $(n_1,n_2)\in \mathbb{E}_\mathcal{K}$ and $f_\mathcal{K}((n_1,n_2))=is\_same\_as$
    \item $(n_2,n_1)\in \mathbb{E}_\mathcal{K}$ and $f_\mathcal{K}((n_2,n_1))=is\_same\_as$
\end{itemize}

\end{definition}

\begin{example}
Let us consider the representation of a piece of knowledge shown in the Figure \ref{fig:ch3_know_ex}, the representation of the sentences in a WSC problem as shown in the Figure \ref{fig:sent1}. Then, according to the Definition \ref{def:nat_res}, following is the value of the set of
node pairs (i.e., $\mathbb{M}$). 

    $\mathbb{M}$ = \{(\textit{weak\_12}, \textit{weak\_9}), (\textit{lift\_5}, \textit{lifts\_4}), (\textit{he\_9}, \textit{person2\_7}),(\textit{man\_2}, \textit{person1\_1}, (\textit{son\_7}, \textit{someone\_5}), (\textit{was\-\_10}, \textit{is\_8}), (\textit{not\_4}, \textit{not\_3}), (\textit{could\_3}, \textit{can\_2})\}. 
% \end{center}
We can see that (\textit{he\_9}, \textit{person2\_7})$\in\mathbb{M}$ and (\textit{man\_2}, \textit{person1\_1})$\in\mathbb{M}$, $(person1\_1,person2\_7)$ is an edge in the graphical representation of the knowledge with label $is\_same\_as$. Then, according to the Definition \ref{def:nat_res} the \textbf{`most natural resolution'} for \textbf{\textit{he\_9}} is \textbf{\textit{man\_2}}.
\end{example}

\begin{definition}[\textbf{Answer of a WSC Problem}] \label{def:ans_def}
    Let $\mathcal{S}$ be a sequence of sentences in a WSC problem $\mathcal{P}$, $\mathbb{T}(S)$ be the set of tokens in $\mathcal{S}$, $p\in \mathbb{T}(S)$ be the token which represents the pronoun to be resolved, $a_1,a_2\in \mathbb{T}(S)$ be two tokens which represent the two answer choices,
    $\mathcal{G_S}=(\mathbb{V}_\mathcal{S},\mathbb{E}_\mathcal{S},f_\mathcal{S})$ be a graphical representation of $\mathcal{S}$, 
    and $\mathcal{G_K}=(\mathbb{V}_\mathcal{K},\mathbb{E}_\mathcal{K},f_\mathcal{K})$ be a graphical representation of a piece of knowledge such that $f_\mathcal{K}$ is defined using $f_\mathcal{S}$. Then,
    \begin{itemize}[wide, nosep, labelindent = 0pt, topsep = 0ex]
        \item $a_1$ is the answer of $\mathcal{P}$, if only $a_1$ provides the `most natural resolution' for $p$,
    
        \item $a_2$ is the answer of $\mathcal{P}$, if only $a_2$ provides the `most natural resolution' for $p$,
        
        \item no answer otherwise
    
    \end{itemize}
\end{definition}

\begin{example}
Let us consider the representation of a piece of knowledge from Figure \ref{fig:ch3_know_ex}, the representation of WSC sentences from Figure \ref{fig:sent1}, the token for pronoun to resolve is \textit{`he\_9'}, the tokens for answer choices are \textit{`man\_2'} and \textit{`son\_8'}. Then according to the Definition \ref{def:nat_res}, only \textit{`man\_2'} provides the `most natural resolution' for \textit{`he\_9'}. Hence, according to the Definition \ref{def:ans_def} \textbf{\textit{`man\_2'}} is the answer of the WSC problem.
\end{example}

\vspace{-10pt}
\subsection{Winograd Schema Challenge Reasoning (WiSCR) Algorithm} \label{sec:reas_algo}
\vspace{-5pt}

\noindent
% \newline
\textbf{Input to the Algorithm: }a graphical representation, $\mathcal{G_S} = (\mathbb{V}_\mathcal{S},\mathbb{E}_\mathcal{S})$, of the sentences in a WSC problem (By Definition \ref{def:text}), a node $p$ in $\mathcal{G_S}$ which represents the pronoun to be resolved, two nodes $a_1$ and $a_2$ in $\mathcal{G_S}$ which represent the two answer choices for the WSC problem, and a graphical representation, $\mathcal{G_K} = (\mathbb{V}_\mathcal{K},\mathbb{E}_\mathcal{K})$, of a commonsense knowledge (By Definition \ref{def:know_rep}).

\noindent
% \newline
\textbf{Output of the Algorithm: }The algorithm outputs $a_1$, $a_2$ or it does not output any answer.

\noindent
% \newline
\textbf{Behavior of the Algorithm: }\\
\textbf{STEP 1: }In this step a subgraph of $\mathcal{G_S}$ is extracted. Let the extracted subgraph be named $\mathcal{G_S}'$. $\mathcal{G_S}'$ contains all the nodes which are not class nodes in $\mathcal{G_S}$. All the edges which connect such nodes are also extracted. An example of the output of the Step 1 is shown in the Figure \ref{fig:ex_algo_step1}. The entire graph is the representation of the sentences in a WSC problem, and the highlighted part of the graph represents the subgraph extracted in this step.

\begin{figure}[ht]
  \centering
%    \reflectbox{%
      \includegraphics[width=0.9\textwidth]{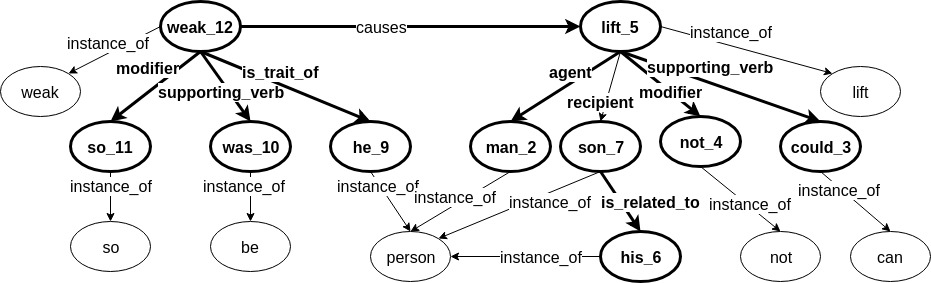}
  \caption{An Example of Step 1 Output of the WiSCR Algorithm with Respect to the WSC Sentence \textit{``The man could not lift his son because he was so weak.''}}
  \label{fig:ex_algo_step1}
\end{figure}

\noindent
\textbf{STEP 2: }In this step a subgraph of $\mathcal{G_K}$ is extracted. Let the extracted subgraph be named $\mathcal{G_K}'$. $\mathcal{G_K}'$ contains all the nodes from $\mathcal{G_K}$ which are not class nodes and it contains all the edges which connect such nodes, except the edges which are labeled as \textit{`is\_same\_as'}. An example of the output of the Step 2 is shown in the Figure \ref{fig:ex_algo_step2}. The entire graph in the figure is the representation of a piece of knowledge (as shown in Figure \ref{fig:ch3_know_ex}) and the highlighted part of the graph is the subgraph extracted in this step.

\begin{figure}[ht]
  \centering
%    \reflectbox{%
      \includegraphics[width=\textwidth]{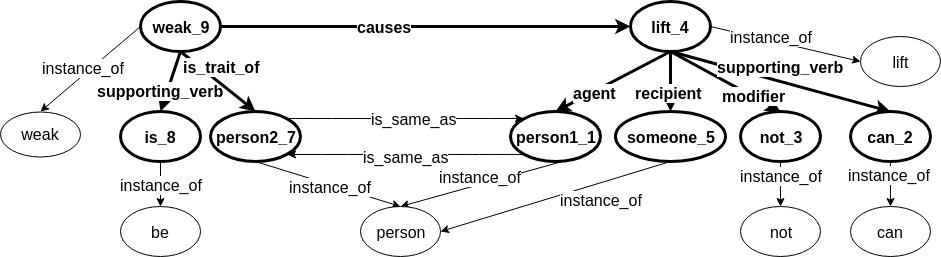}
  \caption{Example of Step 2 Output of the WiSCR Algorithm}
  \label{fig:ex_algo_step2}
  \vspace{-15pt}
\end{figure}

\noindent
% \newline
\textbf{STEP 3: }In this step, all possible graph-subgraph isomorphisms \cite{cordella2004sub} are detected between $\mathcal{G_S}'$ and $\mathcal{G_K}'$ (the subgraphs from the previous two steps respectively). A graph-subgraph isomorphism is a mapping (say $\mathbb{M}$) between two graphs ($\mathcal{G_S}'$ and $\mathcal{G_K}'$) such that $\mathbb{M}$ is a set of pairs of the form $(x,y)$ where $x$ is a node in $\mathcal{G_S}'$, $y$ is a node in $\mathcal{G_K}'$, and if for every $(x,y)\in\mathbb{M}$, $x$ is replaced by $y$ then $\mathcal{G_K}'$ becomes a subgraph of the node replaced $\mathcal{G_S}'$. 
If such a mapping does not exist then $\mathbb{M}=\emptyset$. An important constraint that we put on the mapping set is that for each $(x,y)\in\mathbb{M}$, both $x$ and $y$ must be instances of same class. This is because our assumption for a correct knowledge is that it represents a scenario which is similar to the sentences in the concerned WSC problem. For example if a WSC sentence mentions about \textit{`lift'} action with the help of the word \textit{`lifting'} then a suitable knowledge must also mention about \textit{`lift'} action. It does not matter which form of a word (e.g., \textit{`lifting'} or \textit{`lifts'}) is used in the knowledge or the WSC sentences. This information is captured by the class nodes in the graphical representations.

\noindent
\textbf{STEP 4: }In this step an answer to a WSC problem is deduced from the input representations and the results of the previous steps of this algorithm. For each of the graph-isomorphism detected in Step 3, an answer to the input WSC problem is extracted by using the following rules.
 
\noindent
$\bullet$ The answer choice $a_1$ is an answer with respect to the set $\mathbb{M}$ if 
    $(p,n_1)\in\mathbb{M}$, 
    $(a_1,n_2)\in\mathbb{M}$, 
    either $(n_1,n_2)$ or $(n_2,n_1)$ is a directed edge in $\mathcal{G_K}$ and it is labeled as \textit{`is\_same\_as'}, and there does not exist an $n$ and an $x$ such that $(x,n)\in\mathbb{M}$ and either $(n_1,n)$ or $(n,n_1)$ is an edge in $\mathcal{G_K}$ labeled as \textit{`is\_same\_as'}

\noindent
$\bullet$ The answer choice $a_2$ is an answer with respect to the set $\mathbb{M}$ if 
    $(p,n_1)\in\mathbb{M}$, 
    $(a_2,n_2)\in\mathbb{M}$, 
    either $(n_1,n_2)$ or $(n_2,n_1)$ is a directed edge in $\mathcal{G_K}$ and it is labeled as \textit{`is\_same\_as'}, and there does not exist an $n$ and an $x$ such that $(x,n)\in\mathbb{M}$ and either $(n_1,n)$ or $(n,n_1)$ is an edge in $\mathcal{G_K}$ labeled as \textit{`is\_same\_as'}

    % \item 
\noindent
$\bullet$ Otherwise the input WSC problem does not have an answer with respect to the set $\mathbb{M}$

Finally, after processing all the isomorphisms, if $a_1$ is the only answer retrieved then $a_1$ is the final answer. If $a_2$ is the only answer retrieved then $a_2$ is the final answer. Otherwise the algorithm does not ouput an answer.

\vspace{-2pt}
\begin{theorem}\label{theorem:1}
Let $\mathcal{S}$ be a sequence of sentences in a WSC problem $\mathcal{P}$, $\mathcal{G_S}=(\mathbb{V}_\mathcal{S},\mathbb{E}_\mathcal{S},f_\mathcal{S})$ be a graphical representation of $\mathcal{S}$, $p$ be a node in $\mathcal{G_S}$ such that it represents the pronoun to be resolved in $\mathcal{P}$, $a_1$ and $a_2$ be two nodes in $\mathcal{G_S}$ such that they represent the two answer choices for $\mathcal{P}$, and $\mathcal{G_K}=(\mathbb{V}_\mathcal{K},\mathbb{E}_\mathcal{K},f_\mathcal{K})$ be a graphical representation of a piece of knowledge such that $f_\mathcal{K}$ is defined using $f_\mathcal{S}$. Then, the Winograd Schema Challenge Reasoning (WiSCR) algorithm outputs,

\begin{itemize}[wide, nosep, labelindent = 0pt, topsep = 0ex]
    \item $a_1$ as the answer of $\mathcal{P}$, if only $a_1$ provides the 'most natural resolution' (By Definition \ref{def:nat_res}) for $p$ in $\mathcal{G_S}$,
    
    \item $a_2$ as the answer of $\mathcal{P}$, if only $a_2$ provides the `most natural resolution' for $p$ in $\mathcal{G_S}$,
    
    \item no answer otherwise.
\end{itemize}
\end{theorem}

\subsection{Implementation of the WiSCR Algorithm}
\vspace{-5pt}
There are various constraints imposed on the two input graphs in the WiSCR algorithm to retrieve the final answer. For example, in Step 3 a constraint that both the nodes in a pair belonging to an isomorphism set must be instances of the same class node. Considering that, our main motivation of using ASP to implement the WiSCR algorithm is to make the process of adding new constraints easier. In this section, first we present the details of the ASP encoding of the inputs to WiSCR algorithm and an ASP\ implementation of the WiSCR algorithm. Then we show, with the help of examples, how the current implementation can be easily updated to include new constraints.

\subsubsection{ASP encoding of Inputs} 
\vspace{-5pt}
There are four inputs to the algorithm, a sequence of sentences in a WSC problem, a pronoun to be resolved, two answer choices and a piece of knowledge. The WSC sentences are represented as a graph. Each edge in the graph is encoded in the ASP format by using a ternary predicate \texttt{has\_s(h,l,t)}, where \texttt{h} and \texttt{t} are two nodes and \texttt{l} is an edge label of the directed edge from \texttt{h} to \texttt{t}. Similarly, a piece of knowledge is represented as a graph. It is encoded in ASP by using a ternary predicate \texttt{has\_k(h$_1$,l$_1$,t$_1$)}, where \texttt{h$_1$} and \texttt{t$_1$} are two nodes and \texttt{l$_1$} is an edge label of the directed edge from \texttt{h$_1$} to \texttt{t$_1$}. The pronoun is encoded in ASP by using a unary predicate \texttt{pronoun(p)} where \texttt{p} is the pronoun. Similarly, the two answer choices are encoded by using the unary predicates \texttt{ans\_ch1(a1)} and \texttt{ans\_ch2(a2)}, respectively.

\subsubsection{ASP implementation of the Step 1 of WiSCR Algorithm}
\vspace{-5pt}
In Step 1 of the WiSCR algorithm a subgraph of the graphical representation of WSC sentences is extracted such that the subgraph contains only the non-class nodes and the edges which are not labeled as \textit{instance\_of}. Following ASP rules encode the first step of the WiSCR algorithm. 
\begin{small}
\begin{lstlisting}
|\textbf{s11:}| node_G_s(X) :- has_s(X,R,Y), R!="instance_of".
|\textbf{s12:}| node_G_s(Y) :- has_s(X,R,Y), R!="instance_of".
|\textbf{s13:}| edge_G_s(X,R,Y) :- has_s(X,R,Y), R!="instance_of".
\end{lstlisting}
\end{small}

\texttt{node\_G\_s(X)} represents a node \texttt{X} in the extracted subgraph, \texttt{edge\_G\_s(X,R,Y)} represents an edge, labeled \texttt{R}, between the nodes \texttt{X} and \texttt{Y} in the extracted subgraph. 

\subsubsection{ASP implementation of the Step 2 of WiSCR Algorithm}
\vspace{-5pt}
In Step 2 of the WiSCR algorithm a subgraph of the graphical representation of a piece of knowledge is extracted such that the subgraph contains only the non-class nodes and the edges which are not labeled as \textit{instance\_of} or \textit{is\_same\_as}. Following ASP rules encode the second step of the WiSCR algorithm. 
\begin{small}
\begin{lstlisting}
s21: node_G_k(X) :- has_k(X,R,Y), R!="instance_of".
s22: node_G_k(Y) :- has_k(X,R,Y), R!="instance_of".
s23: edge_G_k(X,R,Y) :- has_k(X,R,Y), R!="instance_of", 
                        R!="is_same_as".
\end{lstlisting}
\end{small}

\texttt{node\_G\_k(X)} represents a node \texttt{X} in the extracted subgraph and \texttt{edge\_G\_k(X,R,Y)} represents an edge, labeled \texttt{R}, between the nodes \texttt{X} and \texttt{Y} in the extracted subgraph.

\subsubsection{ASP implementation of the Step 3 of WiSCR Algorithm}
Let $\mathcal{G}_\mathcal{S}'$ and $\mathcal{G}_\mathcal{K}'$ be the graphs extracted in step 1 and 2 of the WiSCR algorithm respectively. Then, in this step, all possible sets of pairs (say $\mathbb{M}_i$) of the form $(x,y)$ are extracted from $\mathcal{G}_\mathcal{S}'$ and $\mathcal{G}_\mathcal{K}'$ such that $x$ is a node in $\mathcal{G}_\mathcal{S}'$, $y$ is a node in $\mathcal{G}_\mathcal{K}'$, both $x$ and $y$ are instances of the same class and if for every $(x,y)\in\mathbb{M}_i$, $x$ is replaced by $y$ then $\mathcal{G}_\mathcal{K}'$ becomes a subgraph of the node replaced $\mathcal{G}_\mathcal{S}'$. Following ASP rules encode the third step of the WiSCR algorithm. 

\begin{small}
\begin{lstlisting}
s31: { matches(X,Y) : node_G_s(X), node_G_k(Y) }.
s32: :- matches(X,Y), matches(X1,Y), X!=X1.
s33: :- matches(X,Y), matches(X,Y1), Y!=Y1.
s34: k_node_matches(Y) :- matches(X,Y).
s35: :- not k_node_matches(Y), node_G_k(Y).
s36: :- matches(X,Y), has_s(X,"instance_of",C), 
        not has_k(Y,"instance_of",C).
s37: :- edge_G_k(X1,R,Y1), matches(X,X1), matches(Y,Y1), 
        not edge_G_s(X,R,Y).
\end{lstlisting}
\end{small}

\texttt{matches(X,Y)} represents a pair in a $\mathbb{M}_i$. The rule \texttt{s31} above generates all possible groundings of the form \texttt{matches(X,Y)} such that \texttt{X} is a node in the graph extracted in Step 1 and \texttt{Y} is a node in the graph extracted in Step 2. The rules \texttt{s32} and \texttt{s33} only keep the answer sets in which each \texttt{X} in the groundings of \texttt{matches(X,Y)} contains exactly one corresponding \texttt{Y} and vice-versa. The remaining answer sets are removed by the rules \texttt{s32} and \texttt{s33}. The rules \texttt{s34} and \texttt{s35}
removes all the answer sets in which there does not exist a grounding of \texttt{matches(X,Y)} corresponding to each node in the graph extracted in Step 2. The rule \texttt{s36} removes all the answer sets in which at least one grounding of \texttt{matched(X,Y)} exists such that both \texttt{X} and \texttt{Y} are not instances of the same node in the knowledge graph. Finally, the rule \texttt{s37} ensures that if two node \texttt{X} and \texttt{Y} in the graph extracted in the Step 2 match with two nodes \texttt{X1} and \texttt{Y1} respectively in the graph extracted in the Step 1, and \texttt{(X1,R,Y1)} is an edge in the graph from Step 2 then \textit{(X,R,Y)} is an edge in the graph from Step 1.

\subsubsection{Implementation of the Step 4 of WiSCR Algorithm}
In this step an answer to the input WSC problem is retrieved from the inputs of the WiSCR algorithm and the outputs of the steps 1 through 3. There are two parts of this the implementation in this step. The first part uses ASP rules to extract an answer from each set of pairs generated by the ASP implementation of Step 3 of the algorithm. Separate rules are used for each answer choice. Following ASP rules encode this part of Step 4 for the first answer choice. 
\begin{small}
\begin{lstlisting}
s41: invalid_1 :- matches(P,N1), matches(X,N2), ans_ch1(A),
                  pronoun(P), A!=X, N1!=N2, 
                  has_k(N1,"is_same_as",N2).
s42: invalid_2 :- matches(P,N1), matches(X,N2), ans_ch2(A),
                  pronoun(P), A!=X, N1!=N2, 
                  has_k(N1,"is_same_as",N2).
s43: ans(A) :- matches(P,N1), matches(A,N2), ans_ch1(A), 
               not invalid_1, pronoun(P), 
               has_k(N1,"is_same_as",N2).
s44: ans(A) :- matches(P,N1), matches(A,N2), ans_ch2(A),
               not invalid_2, pronoun(P), 
               has_k(N1,"is_same_as",N2).
\end{lstlisting}
\end{small}
Here, \texttt{ans(A1)} represents that \texttt{A1} is an answer of the input WSC problem given a set of \texttt{matches}. Similar rules are written for the second answer choice (assume rules \texttt{s45}, \texttt{s46}, \texttt{s47}, \texttt{s48}). Finally the following rule makes sure that there is one answer generated with respect to one set of \texttt{matches(X,Y)} facts.

\noindent
\texttt{s49: :- ans(A1), ans(A2), A1!=A2.}

The above AnsProlog program produces zero or more answer sets. Zero answer sets mean that none of the sets of \texttt{matches} were able to produce an answer. The second part assembles all the answers and produces the final answer of the input WSC problem. This part of the algorithm is implemented in python. Let us call the python procedure which implements this part as \textsc{AnswerFinder}. \textsc{AnswerFinder} takes as input the answers generated by the ASP code and outputs the final answer based on the following conditions.

\noindent
$\bullet$ if all the answers correspond to one common answer then the algorithm outputs it as final answer,\\
$\bullet$ otherwise the algorithm does not ouput anything.

The WiSCR algorithm requires graph-subgraph isomorphism detection as a sub-module. Graph-subgraph isomorphism\footnote{\url{https://en.wikipedia.org/wiki/Subgraph_isomorphism_problem}} is an NP-Complete problem. In recent times, there has been remarkable progress made in
computing answer sets efficiently. Some of the popular answer set solvers are SModels\footnote{\url{http://www.tcs.hut.fi/Software/smodels/}}, CModels\footnote{\url{http://www.cs.utexas.edu/users/tag/cmodels/}} and Clingo\footnote{\url{http://potassco.sourceforge.net/}}. In this work we used Clingo, which use techniques similar to the ones used in SAT solvers \cite{lin2004assat}. The rest of the steps in the algorithm can be performed in polynomial time.

\subsubsection{Adding New Constraints}

    Suppose we would like to add a constraint that a pair of nodes are valid in a graph-subgraph isomorphism if the two nodes in it are synonyms of each other or they are instances of the same class node. Then we can encode such constraint by replacing the rule \texttt{s36} with the following three rules.
    
    \begin{small}
    \begin{lstlisting}
    valid_pair(X,Y) :- has_s(X,"instance_of",C),
                       has_k(Y,"instance_of",C).
    valid_pair(X,Y) :- synonyms(X,Y).
    :- matches(X,Y), not valid_pair(X,Y).
    \end{lstlisting}
    \end{small}
    
    Here, \texttt{synonyms(X,Y)} represents that a node \texttt{X} in the WSC sentences' graph is synonymous to a node \texttt{Y} in the knowledge graph. We assume that a set of \texttt{synonymous(X,Y)} facts are provided as input. Let us consider the following WSC problem and knowledge as an example to understand the significance of the above rules,
    
    \noindent
    \textbf{Sentence:} The man could not lift his son because \textbf{he$_{pronoun}$} was so weak.\\\textbf{Question:} Who was weak?\textbf{Answer Choices:} a) man b) son.\\\textbf{Knowledge: }
    \textbf{IF} \textit{person1 could not lift someone because person2 was frail} \textbf{THEN} \textit{person1\_1} \textbf{is same as} \textit{person2\_7}
    
    The basic implementation of the WiSCR algorithm will not be able to utilize the above knowledge because the knowledge has the word \textit{frail} instead of \textit{weak}. However since \textit{weak} is a synonym of \textit{frail}, if we provide \texttt{synonyms(weak\_12,frail\_9)} as an input to the code which is updated by replacing the rule \texttt{s36} with the above mentioned three rules then the ASP implementation can handle the knowledge and the algorithm outputs the correct answer, i.e., \textit{man\_2}.
    
    Replacing an existing rule with only three new ones allows the algorithm to be more flexible with respect to the needed knowledge. This also shows how additional constraints and generalizations can be easily expressed as new ASP rules. Another generalization could be done by using similarity along with synonymy to add node pairs in an isomorphism. We say that if the similarity between two nodes is above a certain threshold then allow them to be added to the isomorphism set. An additional rule to encode that would be,
    
    \begin{small}
    \begin{lstlisting}
    valid_pair(X,Y) :- similar(X,Y).
    \end{lstlisting}
    \end{small}
    
    Here, \texttt{similar(X,Y)} represents that a node \texttt{X} in the WSC sentences' graph is similar to a node \texttt{Y} in the knowledge graph. We assume that a set of \texttt{similar(X,Y)} facts are provided as input.

\begin{definition}[\textbf{AnsProlog Program for WiSCR Algorithm}] \label{def:asp_prog}
	Let $\mathcal{S}$ be a sequence of sentences in a WSC problem $\mathcal{P}$, $\mathbb{T}(S)$ be the set of tokens in $\mathcal{S}$, $p\in \mathbb{T}(S)$ be the token which represents the pronoun to be resolved, $a_1,a_2\in \mathbb{T}(S)$ be two tokens which represent the two answer choices,
    $\mathcal{G_S}=(\mathbb{V}_\mathcal{S},\mathbb{E}_\mathcal{S},f_\mathcal{S})$ be a graphical representation of $\mathcal{S}$, 
    and $\mathcal{G_K}=(\mathbb{V}_\mathcal{K},\mathbb{E}_\mathcal{K},f_\mathcal{K})$ be a representation of a piece of knowledge such that $f_\mathcal{K}$ is defined using $f_\mathcal{S}$. Then, we say that the AnsProlog program $\Pi(\mathcal{G}_S,\mathcal{G}_K, p, a_1, a_2)$ is the answer set program consisting of %the following:
	\begin{enumerate}[(i), wide, nosep, labelindent = 0pt, topsep = 0ex]
		\item 
		the facts of the form $has\_s(h_1,l_1,t_1)$ and $has\_k(h_2,l_2,t_2)$,
		\item 
		a fact of the form $pronoun(p)$,
		\item 
		two facts of the form $ans\_ch1(a_1)$ and $ans\_ch2(a_2)$,
		\item 
		the rules \texttt{s11} to \texttt{s49}
	\end{enumerate}
\end{definition}

\begin{theorem}\label{theorem:2}
	Let $\mathcal{S}$ be a sequence of sentences in a WSC problem $\mathcal{P}$, $\mathbb{T}(S)$ be the set of tokens in $\mathcal{S}$, $p\in \mathbb{T}(S)$ be the token which represents the pronoun to be resolved, $a_1,a_2\in \mathbb{T}(S)$ be two tokens which represent the two answer choices,
    % $\mathcal{P}$ be a WSC problem
    $\mathcal{G_S}=(\mathbb{V}_\mathcal{S},\mathbb{E}_\mathcal{S},f_\mathcal{S})$ be a graphical representation of $\mathcal{S}$,  
    and $\mathcal{G_K}=(\mathbb{V}_\mathcal{K},\mathbb{E}_\mathcal{K},f_\mathcal{K})$ be a representation of a piece of knowledge such that $f_\mathcal{K}$ is defined using $f_\mathcal{S}$. Also, $\Pi(\mathcal{G}_S,\mathcal{G}_K, p, a_1, a_2)$ be the AnsProlog program for WiSCR algorithm and \textsc{AnswerFinder} be the python procedure defined in Section 4.2.5. Then, the WiSCR algorithm produces an answer $x$ to the input WSC problem iff 	$\Pi(\mathcal{G}_S,\mathcal{G}_K, p, a_1, a_2)$ and \textsc{AnswerFinder} together output the answer $x$.
    
\end{theorem}

\section{Experimental Evaluation of the WiSCR Algorithm}
\vspace{-5pt}
The main goal of the evaluation process is to validate if the WiSCR algorithm is able to correctly answer the WSC problems if the problem and a relevant knowledge is provided as inputs to it in the specified formats. We evaluated a corpus\footnote{Avaiable at \url{https://tinyurl.com/y22ykz5p}} of 291 WSC problems. In this section we present the three experiments which we performed to validate the WiSCR algorithm and our findings with respect to those experiments. 

\paragraph{\textbf{Experiment 1: }}In this experiment we manually created the input graphical representations of the WSC sentences and the needed knowledge. We found that the WSC problems require different kinds of knowledge. The knowledge defined in this work (See Definition \ref{def:know}) is helpful in tackling 240 out of 291 WSC problems (82.47\%). So we wrote the representations for those 240 problems by hand. The ASP implementation answered all of those problems correctly. The reasoning algorithm defined in this work relies on the fact that the provided knowledge contains the same or similar scenarios as that of the original WSC sentences. A scenario is basically defined by the actions, properties and the type of entities present. By performing a comprehensive analysis of the WSC problems, we found that 240 out of 291 WSC problems can be answered using such knowledge. The remaining problems require two different kinds of knowledge. 26 problems require multiple pieces of knowledge. For example, \textbf{WSC Sentence: }\textit{Mary tucked her daughter Anne into bed, so that she could work.} \textbf{Question:} Who is going to work?  \textbf{Knowledge 1: }\textit{someone who is tucked into bed, may sleep} \textbf{Knowledge 2: }\textit{someone who's daughter is sleeping may be able to work}. It was observed that such knowledge has a partial overlap with the scenarios in a WSC problem. For example see the WSC sentence and knowledge 1 shown above. Due to this, such knowledge is not handled by the current algorithm. If one tries to format such knowledge according to the Definition \ref{def:know} then the reasoning algorithm will not answer anything because it will not be able to find a graph-subgraph isomorphism between the subgraphs of WSC sentences' representation and knowledge's representation. The remaining 25 problems require the knowledge that one statement is more likely to be true than the other. For example, \textbf{WSC Sentence: }\textit{Sam tried to paint a picture of shepherds with sheep, but they ended up looking more like dogs.} \textbf{Question: }\textit{What looked like dogs?} \textbf{Knowledge: } \textit{Sheep looks like a dog} \textbf{is more likely to be true than} \textit{Shepherd looks like a dog}. Such knowledge is also not handled by the current reasoning algorithm because it does not satisfy the definition (Def \ref{def:know}) of knowledge reasoned with in this work. A list of the WSC problems which are not handled by the WiSCR algorithm because of the reasons mentioned above is also present at  \url{https://tinyurl.com/y22ykz5p}.

\paragraph{\textbf{Experiment 2: }}In this experiment we considered the 240 WSC problems that are handled by the WiSCR algorithm. The needed knowledge for all the 240 problems was manually written in the \textit{`\textbf{IF} S \textbf{THEN} x \textbf{is same as} y'} format as mentioned in the Definition \ref{def:know}. Both, the WSC problems and the needed knowledge were automatically converted into graphs by using two K-Parser wrappers. The details of the K-Parser wrappers are provides in the paragraph below. 200 (82.98\%) out of 240 problems were correctly answered in this experiment by the WiSCR algorithm. The remaining 40 problems were not answered because of syntactic dependency parsing errors and part-of-speech errors while generating the representations.

Two wrappers over K-Parser were developed as part of this work. The first translates a sequence of sentences into a graphical representation that satisfies the Definition \ref{def:text}. K-Parser produces a graph for an input English sequence of sentences. The only two differences between the K-Parser output and the representation in Definition \ref{def:text} is that in K-Parser output there are two levels of class nodes instead of one and the K-Parser output contains semantic roles of entities. So, as part of this wrapper the two levels of classes was reduced to one by keeping the \textit{superclasses} of \textit{noun} and \textit{pronoun} words and by keeping the classes which represent the lemmatized form of other types of nodes. The semantic roles are not considered by the wrapper. The second wrapper is used to translate a knowledge of the form \textit{\textbf{IF} S \textbf{THEN} x \textbf{is same as} y} where \textit{S} is a sentence and \textit{x}, \textit{y} are tokens in \textit{S}. In this wrapper the same modifications to the K-Parser output of \textit{S} are made as were in the wrapper 1 along with the addition of two extra edges. An edge from the node representing \textit{x} to a node representing \textit{y} was added and labeled as \textit{is\_same\_as} and another edge from \textit{y} to \textit{x} with same label is also added.

\paragraph{\textbf{Experiment 3:} }In this experiment we used a technique to automatically extract the knowledge that is needed for the WSC problems which were correctly represented by using K-Parser. The knowledge was found and automatically extracted for 120 problems. The ASP implementation was able to correctly answer all of the 120 problems. The automated extraction of knowledge is inspired from the work done in \cite{sharma2015towards}. The idea there is to extract a set of sentences (by using a search engine) which are similar to the original WSC sentences in terms of the actions and properties in it. %The scenarios in a text are mainly defined by the actions and properties in it. 
Such sentences are then parsed with the help of K-Parser to extract the knowledge. For example, a sentence extracted for the Winograd sentence shown in Figure \ref{fig:sent1} is \textit{``She could not lift him because she is weak.''}. And the knowledge extracted from the above sentence is \textit{``IF person1 could not lift someone because person2 is weak THEN person1\_1 is same as person2\_7''}. Because of the limited availability of search engine access, the sentences similar to only 120 WSC sentences could be extracted. Those sentences are then passed to a rule based knowledge extraction module. The module uses the K-Parser outputs to find the patterns which satisfy the kind of knowledge handled by our reasoning algorithm.

\section{Related Work}
Over the years various approaches have been proposed to solve the Winograd Schema Challenge by using additional knowledge. Such works include the ones which focus on defining the reasoning theories \cite{bailey2015winograd,schuller2014tackling,richard2018role,wolff2018interpreting}. These approaches mention the need of additional knowledge and reasoning, but they suffer from the issue of low coverage on the WSC corpus.

Another set of approaches address the knowledge extraction and reasoning with it in a joint method. Such approaches include the ones which use on the fly knowledge extraction \cite{sharma2015towards,emami2018knowledge}, and the ones which perform knowledge extraction with respect to a pre-populated knowledge base \cite{isaak2016tackling}. These approaches rely on the heuristic procedures. More recently, composition embedding \cite{liu2017cause} and statistical language modelling \cite{radford2019language} based approaches have been used to address the challenge. The later recently reported the state of the art accuracy (70.70\%) on the overall corpus. Such approaches try to capture the knowledge in the form of word and sentences embedding and later use it to infer which phrase is more probable. This helps in the cases where the needed knowledge is based on the possible correlation between two terms for example \textit{``a ball is kicked''} where there is a correlation between \textit{kicked} and \textit{ball}. But it is not be able to infer that \textit{``worm is tasty''} for the Winograd Schema Challenge problem \textit{``Fish ate the worm. It was tasty.''}. On the other hand it is more possible that it finds \textit{``fish is tasty''} more probable because \textit{``fish''} and \textit{``tasty''} has higher chances of occurring in the same context in text corpora.

\section{Conclusion}
In this work, we attempted to solve the Winograd Schema Challenge by reasoning with additional knowledge. To that end we defined a graphical representation of the English sentences in the input problems and a graphical representation of the relevant knowledge. We also defined a commonsense reasoning algorithm for WSC (WiSCR algorithm). We showed how an approach built on top of graph-subgraph isomorphism encoded in ASP is able to tackle 240 out of 291 WSC problems. We presented how the ASP implementation of the algorithm allows us to add new constraints easily. It also makes the current implementation to easily generalize by adding new rules. 

\bibliographystyle{acmtrans}
\bibliography{refs}

\newpage
\section{APPENDIX}

\subsection{Proof of Theorem \ref{theorem:1}}
The proof of Theorem \ref{theorem:1} is done using a set of lemmas. In this sections we present those lemmas and then use them to prove Theorem \ref{theorem:1}.

\begin{lemma}\label{lemma:1} Let $\mathcal{G_S}=(\mathbb{V}_\mathcal{S},\mathbb{E}_\mathcal{S},f_\mathcal{S})$ be a graphical representation of the sequence of sentences in a WSC problem. Then, Step 1 of the WiSCR Algorithm extracts a subgraph $\mathcal{G_\mathcal{S}'}$ of $\mathcal{G_S}$ such that $\mathcal{G_\mathcal{S}'} = (\mathbb{V}_\mathcal{S}',\mathbb{E}_\mathcal{S}',f_\mathcal{S}')$ where $\mathbb{V}_\mathcal{S}' = \mathbb{V}_\mathcal{S} - \mathbb{V}_\mathcal{S}^c$, $\mathbb{V}_\mathcal{S}^c$ is a set of all the class nodes in $\mathcal{G_S}$, $f_\mathcal{S}' = f_\mathcal{S}$, $\mathbb{E}_\mathcal{S}'=\mathbb{E}_\mathcal{S}-\mathbb{E}_\mathcal{S}^c$, and $e\in \mathbb{E}_\mathcal{S}^c$ if $f(e) = instance\_of$.
\end{lemma}

\begin{proof}
According to the Step 1 of the WiSCR algorithm, given a graph $\mathcal{G_S}=(\mathbb{V}_\mathcal{S},\mathbb{E}_\mathcal{S},f_\mathcal{S})$, a subgraph of it is extracted. Let $\mathcal{G_S'} = (\mathbb{V}_\mathcal{S}',\mathbb{E}_\mathcal{S}',f_\mathcal{S}')$ be the extracted subgraph. $\mathbb{V}_\mathcal{S}'$ contains all the nodes from $\mathcal{G_S}$ which are not class nodes, i.e., $\mathbb{V}_\mathcal{S}' = \mathbb{V}_\mathcal{S} - \mathbb{V}_\mathcal{S}^c$, $\mathbb{V}_\mathcal{S}^c$ is a set of all the class nodes in $\mathcal{G_S}$. Also, $\mathbb{E}_\mathcal{S}'$ contains all the edges between the nodes in $\mathbb{V}_\mathcal{S}'$. So, by Definition \ref{def:text} $\mathbb{E}_\mathcal{S}'=\mathbb{E}_\mathcal{S}-\mathbb{E}_\mathcal{S}^c$ where $e\in \mathbb{E}_\mathcal{S}^c$ if $f(e) = instance\_of$. Furthermore, no new edges or nodes are added to $\mathcal{G_S'}$ so $f_\mathcal{S}'$ = $f_\mathcal{S}$.

Hence, the step 1 of the WiSCR Algorithm extract a subgraph $\mathcal{G_S'}$ from $\mathcal{G_S}$ such that if $\mathcal{G_S} = (\mathbb{V}_\mathcal{S},\mathbb{E}_\mathcal{S},f_\mathcal{S})$ then $\mathcal{G_S'} = (\mathbb{V}_\mathcal{S}',\mathbb{E}_\mathcal{S}',f_\mathcal{S}')$ where $\mathbb{V}_\mathcal{S}' = \mathbb{V}_\mathcal{S} - \mathbb{V}_\mathcal{S}^c$, $\mathbb{V}_\mathcal{S}^c$ is a set of all the class nodes in $\mathcal{G_S}$, $f_\mathcal{S}' = f_\mathcal{S}$, $\mathbb{E}_\mathcal{S}'=\mathbb{E}_\mathcal{S}-\mathbb{E}_\mathcal{S}^c$, and $e\in \mathbb{E}_\mathcal{S}^c$ iff $f(e) = instance\_of$.
\end{proof}

\begin{lemma}\label{lemma:2} Let $\mathcal{G_K}=(\mathbb{V}_\mathcal{K},\mathbb{E}_\mathcal{K},f_\mathcal{K})$ be a graphical representation of a knowledge (By Definition \ref{def:know_rep}). Then, Step 2 of the WiSCR Algorithm extracts a subgraph $\mathcal{G_K'}$ from $\mathcal{G_K}$ such that $\mathcal{G_K'} = (\mathbb{V}_\mathcal{K}',\mathbb{E}_\mathcal{K}',f_\mathcal{K}')$ where $\mathbb{V}_\mathcal{K}' = \mathbb{V}_\mathcal{K} - \mathbb{V}_\mathcal{K}^c$, $\mathbb{V}_\mathcal{K}^c$ is a set of all the class nodes in $\mathcal{G_K}$, $f_\mathcal{K}' = f_\mathcal{K}$, $\mathbb{E}_\mathcal{K}'=\mathbb{E}_\mathcal{K}-\mathbb{E}_\mathcal{K}^c$, and $e\in \mathbb{E}_\mathcal{K}^c$ if $f(e) \in \{instance\_of,is\_same\_as\}$.
\end{lemma}

\begin{proof}
According to the Step 2 of the WiSCR algorithm, given a graphical representation of a knowledge $\mathcal{G_K}=(\mathbb{V}_\mathcal{K},\mathbb{E}_\mathcal{K},f_\mathcal{K})$, a subgraph of it is extracted. Let $\mathcal{G_K'} = (\mathbb{V}_\mathcal{K}',\mathbb{E}_\mathcal{K}',f_\mathcal{K}')$ be the extracted subgraph. 
% Then, according to the step 2 of the algorithm, 
$\mathbb{V}_\mathcal{K}'$ contains all the nodes from $\mathcal{G_K}$ which are not class nodes, i.e., $\mathbb{V}_\mathcal{K}' = \mathbb{V}_\mathcal{K} - \mathbb{V}_\mathcal{K}^c$, $\mathbb{V}_\mathcal{K}^c$ is a set of all the class nodes in $\mathcal{G_K}$. Also, $\mathbb{E}_\mathcal{K}'$ contains all the edges between the nodes in $\mathbb{V}_\mathcal{K}'$ except the ones labeled as \textit{`is\_same\_as'}. So, by Definition \ref{def:know_rep} $\mathbb{E}_\mathcal{K}'=\mathbb{E}_\mathcal{K}-\mathbb{E}_\mathcal{K}^c$, and $e\in \mathbb{E}_\mathcal{K}^c$ if $f(e) \in \{instance\_of,is\_same\_as\}$. Furthermore, no new edges or nodes are added to $\mathcal{G_K'}$ so $f_\mathcal{K}'$ = $f_\mathcal{K}$.

Hence, the step 2 of the WiSCR Algorithm extract a subgraph $\mathcal{G_K'}$ from $\mathcal{G_K}$ such that if $\mathcal{G_K} = (\mathbb{V}_\mathcal{K},\mathbb{E}_\mathcal{K},f_\mathcal{K})$ then $\mathcal{G_K'} = (\mathbb{V}_\mathcal{K}',\mathbb{E}_\mathcal{K}',f_\mathcal{K}')$ where $\mathbb{V}_\mathcal{K}' = \mathbb{V}_\mathcal{K} - \mathbb{V}_\mathcal{K}^c$, $\mathbb{V}_\mathcal{K}^c$ is a set of all the class nodes in $\mathcal{G_K}$, $f_\mathcal{K}' = f_\mathcal{K}$, $\mathbb{E}_\mathcal{K}'=\mathbb{E}_\mathcal{K}-\mathbb{E}_\mathcal{K}^c$, and $e\in \mathbb{E}_\mathcal{K}^c$ if $f(e) \in \{instance\_of,is\_same\_as\}$.
\end{proof}

% \todo{REDO BELOW PROP}
% \begin{proposition}\label{prop:1}
% 	\cite{cordella2004sub}
\begin{lemma}\label{lemma:isomorphism}
Let $\mathcal{G_S}$ = ($\mathbb{V}_\mathcal{S}$,$\mathbb{E}_\mathcal{S}$,$f_\mathcal{S}$) be a graphical representation of a sequence of sentences in a WSC problem, $\mathcal{G_S'}$ = ($\mathbb{V}_\mathcal{S}'$,$\mathbb{E}_\mathcal{S}'$,$f_\mathcal{S}'$) be a subgraph of $\mathcal{G_S}$ such that $\mathbb{V}_\mathcal{S}'=\mathbb{V}_\mathcal{S}-\mathbb{V}_\mathcal{S}^c$ where $\mathbb{V}_\mathcal{S}^c$ is the set of all the class nodes in $\mathcal{G_S}$, $f_\mathcal{S}'=f_\mathcal{S}$ and $\mathbb{E}_\mathcal{S}'=\mathbb{E}_\mathcal{S}-\mathbb{E}_\mathcal{S}^c$ where $e\in \mathbb{E}_\mathcal{S}^c$ iff $f_\mathcal{S}(e)=``instance\_of"$. Let $\mathcal{G_K}$ = ($\mathbb{V}_\mathcal{K}$,$\mathbb{E}_\mathcal{K}$,$f_\mathcal{K}$) be a graphical representation of a knowledge where $f_\mathcal{K}$ is defined using $f_\mathcal{S}$, $\mathcal{G_K'}$ = ($\mathbb{V}_\mathcal{K}'$,$\mathbb{E}_\mathcal{K}'$,$f_\mathcal{K}'$) be a subgraph of $\mathcal{G_K}$ such that $\mathbb{V}_\mathcal{K}'=\mathbb{V}_\mathcal{K}-\mathbb{V}_\mathcal{K}^c$ where $\mathbb{V}_\mathcal{K}^c$ is the set of all the class nodes in $\mathcal{G_K}$, $f_\mathcal{K}'=f_\mathcal{K}$ and $\mathbb{E}_\mathcal{K}' = \mathbb{E}_\mathcal{K}-\mathbb{E}_\mathcal{K}^c$ where $e\in \mathbb{E}_\mathcal{K}^c$ iff $f_\mathcal{K}(e) \in \{is\_same\_as,instance\_of\}$. Then, Step 3 of the WiSCR algorithm extracts the all possible sets of node pairs of the form ($a$,$b$) such that either there does not exist such a non-empty set or if $\mathbb{M}_i$ is one such non-empty set then,%all of the below conditions are satisfied or the sets are empty. %Let $\mathbb{M}_i$ be a such non empty set. Then,
\begin{itemize}[wide, nosep, labelindent = 15pt, topsep = 2ex]
    \item for each $(a,b)\in \mathbb{M}_i$, $a\in \mathbb{V}_\mathcal{S}'$ and $b\in \mathbb{V}_\mathcal{K}'$, 
    
    \item for each $(a,b)\in \mathbb{M}_i$, $a$ and $b$ are instances of same class, i.e., $(a,i)\in \mathbb{E}_\mathcal{S}$, $(b,i)\in \mathbb{E}_\mathcal{K}$, $f_\mathcal{S}((a,i))=instance\_of$ and $f_\mathcal{K}((b,i))=instance\_of$
    
    % ($a$,$instance\_of$,$i$)$\in$ $\mathcal{G_S}$ and ($b$,$instance\_of$,$i$)$\in$ $\mathcal{G_K}$, i.e., both $a$ and $b$ are instances of same class, and 
    
    \item if for every pair ($a$,$b$)$\in$ $\mathbb{M}_i$, $a$ is replaced by $b$ in $\mathbb{V}_\mathcal{S}'$ then $\mathcal{G_K'}$ becomes a 
    %induced\footnote{http://mathworld.wolfram.com/Vertex-InducedSubgraph.html} 
    subgraph of the node-replaced $\mathcal{G_S'}$
\end{itemize}
\end{lemma}
\begin{proof}
(i) Given a graphical representation of the sentences in a WSC problem (say $\mathcal{G_S}$ = ($\mathbb{V}_\mathcal{S}$,$\mathbb{E}_\mathcal{S}$,$f_\mathcal{S}$)) and Lemma \ref{lemma:1}, the Step 1 of the WiSCR algorithm produces a subgraph of $\mathcal{G_S}$ (say $\mathcal{G_S'}$ = ($\mathbb{V}_\mathcal{S}'$,$\mathbb{E}_\mathcal{S}'$,$f_\mathcal{S}'$)) such that $\mathbb{V}_\mathcal{S}'=\mathbb{V}_\mathcal{S}-\mathbb{V}_\mathcal{S}^c$ where $\mathbb{V}_\mathcal{S}^c$ is a set of all the class nodes in $\mathcal{G_S}$, $f_\mathcal{S}'=f_\mathcal{S}$ and $\mathbb{E}_\mathcal{S}'=\mathbb{E}_\mathcal{S}-\mathbb{E}_\mathcal{S}^c$ where $e\in \mathbb{E}_\mathcal{S}^c$ if $f_\mathcal{S}(e)=instance\_of$.

\noindent
\newline
(ii) Given a graphical representation of a knowledge (say $\mathcal{G_K}$ = ($\mathbb{V}_\mathcal{K}$,$\mathbb{E}_\mathcal{K}$,$f_\mathcal{K}$)) and Lemma \ref{lemma:2}, the step 2 of the WiSCR algorithm produces a subgraph of $\mathcal{G_K}$ (say $\mathcal{G_K'}$ = ($\mathbb{V}_\mathcal{K}'$,$\mathbb{E}_\mathcal{K}'$,$f_\mathcal{K}'$)) such that $\mathbb{V}_\mathcal{K}'=\mathbb{V}_\mathcal{K}-\mathbb{V}_\mathcal{K}^c$ where $\mathbb{V}_\mathcal{K}^c$ is a set of all the class nodes in $\mathcal{G_K}$, $f_\mathcal{K}'=f_\mathcal{K}$ and $\mathbb{E}_\mathcal{K}'=\mathbb{E}_\mathcal{K}-\mathbb{E}_\mathcal{K}^c$ where $e\in \mathbb{E}_\mathcal{K}^c$ if $f_\mathcal{K}(e)\in \{instance\_of,is\_same\_as\}$.

\noindent
\newline
(iii) Given $\mathcal{G_S'}$ and $\mathcal{G_K'}$ are the graphs generated by the steps 1 and 2 of the WiSCR algorithm %(See Lemma \ref{lemma:1} and \ref{lemma:2}) 
respectively, then according to the Step 3 of the WiSCR algorithm, it extracts all possible graph-subgraph isomorphisms between $\mathcal{G_S'}$ and $\mathcal{G_K'}$. In other words, it extracts all possible sets of pairs of the form  $(a,b)$ such that either there does not exist such a non-empty set or if $\mathbb{M}_i$ is one such non-empty set then,
\begin{itemize}[wide, nosep, labelindent = 15pt, topsep = 2ex]
    \item for each $(a,b)\in \mathbb{M}_i$, $a\in\mathbb{V}_\mathcal{S}'$ and $b\in\mathbb{V}_\mathcal{K}'$, 
    
    \item for each $(a,b)\in \mathbb{M}_i$, $a$ and $b$ are instances of same class, i.e., $(a,i)\in\mathbb{E}_\mathcal{S}$, $(b,i)\in\mathbb{E}_\mathcal{K}$, $f_\mathcal{S}((a,i))$ = $instance\_of$ and $f_\mathcal{K}((a,i)) =instance\_of$, and 
    
    \item if for every pair $(a,b)\in\mathbb{M}_i$, $a$ is replaced by $b$ in $\mathbb{V}_\mathcal{S}'$ then $\mathcal{G_K'}$ becomes a subgraph of the node-replaced $\mathcal{G_S'}$
\end{itemize}

\end{proof}

\begin{customthm}{1}
Let $\mathcal{S}$ be a sequence of sentences in a WSC problem $\mathcal{P}$, $\mathcal{G_S}=(\mathbb{V}_\mathcal{S},\mathbb{E}_\mathcal{S},f_\mathcal{S})$ be a graphical representation of $\mathcal{S}$, $p$ be a node in $\mathcal{G_S}$ such that it represents the pronoun to be resolved in $\mathcal{P}$, $a_1$ and $a_2$ be two nodes in $\mathcal{G_S}$ such that they represent the two answer choices for $\mathcal{P}$, and $\mathcal{G_K}=(\mathbb{V}_\mathcal{K},\mathbb{E}_\mathcal{K},f_\mathcal{K})$ be a graphical representation of a knowledge such that $f_\mathcal{K}$ is defined using $f_\mathcal{S}$. Then, the Winograd Schema Challenge Reasoning (WiSCR) algorithm outputs,% an answer of $\mathcal{P}$ such that it satisfies the Definition \ref{def:ans_def}.

\begin{itemize}[wide, nosep, labelindent = 15pt, topsep = 2ex]
    \item $a_1$ as the answer of $\mathcal{P}$, if only $a_1$ provides the 'most natural resolution' (By Definition \ref{def:nat_res}) for $p$ in $\mathcal{G_S}$,
    
    \item $a_2$ as the answer of $\mathcal{P}$, if only $a_2$ provides the `most natural resolution' for $p$ in $\mathcal{G_S}$,
    
    \item No answer otherwise%. %if none of the above two conditions are satisfied
\end{itemize}
\end{customthm}

\begin{proof}
If $\mathcal{G_S}=(\mathbb{V}_\mathcal{S},\mathbb{E}_\mathcal{S},f_\mathcal{S})$ is a graphical representation of the sequence of sentences in a WSC problem then by Lemma \ref{lemma:1}, we have that\\
Step 1 of the WiSCR Algorithm extract a subgraph $\mathcal{G_\mathcal{S}'}$ from $\mathcal{G_S}$ such that $\mathcal{G_\mathcal{S}'} = (\mathbb{V}_\mathcal{S}',\mathbb{E}_\mathcal{S}',f_\mathcal{S}')$ where $\mathbb{V}_\mathcal{S}' = \mathbb{V}_\mathcal{S} - \mathbb{V}_\mathcal{S}^c$, $\mathbb{V}_\mathcal{S}^c$ is a set of all the class nodes in $\mathcal{G_S}$, $f_\mathcal{S}' = f_\mathcal{S}$, $\mathbb{E}_\mathcal{S}'=\mathbb{E}_\mathcal{S}-\mathbb{E}_\mathcal{S}^c$, and $e\in \mathbb{E}_\mathcal{S}^c$ if $f(e) = instance\_of$.\\

\noindent
If $\mathcal{G_K}=(\mathbb{V}_\mathcal{K},\mathbb{E}_\mathcal{K},f_\mathcal{K})$ is a graphical representation of a knowledge then by Lemma \ref{lemma:2}, we have that\\
Step 2 of the WiSCR Algorithm extracts a subgraph $\mathcal{G_K'}$ from $\mathcal{G_K}$ such that $\mathcal{G_K'} = (\mathbb{V}_\mathcal{K}',\mathbb{E}_\mathcal{K}',f_\mathcal{K}')$ where $\mathbb{V}_\mathcal{K}' = \mathbb{V}_\mathcal{K} - \mathbb{V}_\mathcal{K}^c$, $\mathbb{V}_\mathcal{K}^c$ is a set of all the class nodes in $\mathcal{G_K}$, $f_\mathcal{K}' = f_\mathcal{K}$, $\mathbb{E}_\mathcal{K}'=\mathbb{E}_\mathcal{K}-\mathbb{E}_\mathcal{K}^c$, and $e\in \mathbb{E}_\mathcal{K}^c$ if $f(e) \in \{instance\_of,is\_same\_as\}$.\\

\noindent
If $\mathcal{G_S}$, $\mathcal{G_S}'$, $\mathcal{G_K}$ and $\mathcal{G_K}'$ are inputs to the Step 3 of the WiSCR algorithm then by Lemma \ref{lemma:isomorphism}, we have that\\
Step 3 of the WiSCR algorithm produces all the possible sets of node pairs of the form $(a,b)$ such that either there does not exist such a non-empty set or if $\mathbb{M}_i$ is one such non-empty set then,
\begin{itemize}[wide, nosep, labelindent = 15pt, topsep = 2ex]
    \item for each $(a,b)\in \mathbb{M}_i$, $a\in \mathbb{V}_\mathcal{S}'$ and $b\in \mathbb{V}_\mathcal{K}'$, and
    
    \item for each $(a,b)\in \mathbb{M}_i$, $a$ and $b$ are instances of same class, i.e., $(a,i)\in\mathbb{E}_\mathcal{S}$, $(b,i)\in\mathbb{E}_\mathcal{K}$, $f_\mathcal{S}((a,i))=instance\_of$ and $f_\mathcal{K}((a,i))=instance\_of$
    % ($a$,$instance\_of$,$i$)$\in$ $\mathcal{G_S}$ and ($b$,$instance\_of$,$i$)$\in$ $\mathcal{G_K}$)
    
    \item if for every pair ($a$,$b$)$\in$ $\mathbb{M}_i$, $a$ is replaced by $b$ in $\mathbb{V}_\mathcal{S}'$ then $\mathcal{G_K'}$ becomes an induced subgraph of $\mathcal{G_S'}$
\end{itemize}

% are graphical representations of the sentences in a WSC problem and a knowledge respectively, then by Lemma \ref{lemma:3} Step 3 of the WiSCR algorithm produces a set of pairs $\mathbb{M}$ such that either $\mathbb{M}= \emptyset$ or $\mathbb{M}$ is of the form  ($a$,$b$) and when for every pair ($a$,$b$)$\in$ $\mathbb{M}$, $a$ is replaced by $b$, then $\mathcal{G_K'}$

% Given a graphical representation of the sentences in an input WSC problem, i.e., $\mathcal{G_S}$, a graphical representation of an input knowledge, i.e., $\mathcal{G_K}$ and Lemma \ref{lemma:3}, step 3 of the WiSCR algorithm produces a set of pairs $\mathbb{M}$ such that either $\mathbb{M}= \emptyset$ or $\mathbb{M}$ is of the form  ($a$,$b$) and when for every pair ($a$,$b$)$\in$ $\mathbb{M}$, $a$ is replaced by $b$, then $\mathcal{G_K'}$ (a subgraph of $\mathcal{G_K}$ extracted by step 2 of WiSCR algorithm, see Lemma \ref{lemma:2}) becomes an induced subgraph of $\mathcal{G_S'}$ (a subgraph of $\mathcal{G_S}$ extracted in step 1 of the WiSCR algorithm, see Lemma \ref{lemma:1}). 

\noindent
If $p\in\mathbb{V}_\mathcal{S}$ represents the pronoun to be resolved, $a_1,a_2\in\mathbb{V}_\mathcal{S}$ represent the two answer choices. Then by Step 4 of the WiSCR algorithm and for each possible non-empty set of pairs (say $\mathbb{M}_i$) produced by Step 3, we have that
% According to the step 4 of the WiSCR algorithm following are the possible cases,
% Also, it is given that $p$ is a node in $\mathcal{G_S}$ such that it represents the pronoun to be resolved, $a1$ and $a2$ are two nodes in $\mathcal{G_S}$ such that they represent the two answer choices for the input WSC problem. According to the step 4 of the WiSCR algorithm following are the possible cases,
\begin{enumerate}[wide, nosep, labelindent = 5pt, topsep = 1ex]
    \item $a_1$ as an answer if,
    \begin{itemize}[wide, nosep, labelindent = 15pt, topsep = 1ex]
        \item $(p,n_1)\in\mathbb{M}_i$, 
        \item $(a_1,n_2)\in\mathbb{M}_i$, 
        \item $(n_1,n_2)\in\mathbb{E}_\mathcal{K}$ and $f_\mathcal{K}((n_1,n_2))=is\_same\_as$, or $(n_2,n_1)\in\mathbb{E}_\mathcal{K}$ and $f_\mathcal{K}((n_2,n_1))=is\_same\_as$, and
        \item there does not exist an $n$ and an $x$ ($x\neq a_1$) such that $(x,n)\in\mathbb{M}_i$ and either $f_\mathcal{K}((n,n_1))=is\_same\_as$ or $f_\mathcal{K}((n_1,n))=is\_same\_as$.
    \end{itemize}
    
    \item $a_2$ as an answer if,
    \begin{itemize}[wide, nosep, labelindent = 15pt, topsep = 2ex]
        \item $(p,n_1)\in\mathbb{M}_i$, 
        \item $(a_2,n_2)\in\mathbb{M}_i$, 
        \item $(n_1,n_2)\in\mathbb{E}_\mathcal{K}$ where $f_\mathcal{K}((n_1,n_2))=is\_same\_as$, or $(n_2,n_1)\in\mathbb{E}_\mathcal{K}$ where $f_\mathcal{K}((n_2,n_1))=is\_same\_as$, and
        \item there does not exist an $n$ and an $x$ ($x\neq a_2$) such that $(x,n)\in\mathbb{M}_i$ and either $f_\mathcal{K}((n,n_1))=is\_same\_as$ or $f_\mathcal{K}((n_1,n))=is\_same\_as$.
    \end{itemize}
    
    \item not answer is produced if neither $a_1$ nor $a_2$ are found as an answer
\end{enumerate}

Then, the Step 4 of the WiSCR algorithm outputs $a_1$ as the final answer if only $a_1$ is found as an answer with respect to the possible node pairs extracted in the Step 3. The Step 4 of the WiSCR algorithm outputs $a_2$ as the final answer if only $a_2$ is found as an answer with respect to the possible node pairs extracted in the Step 3. The Step 4 of the algorithm does not answer anything otherwise. 

\noindent
By definition of `most natural resolution' and above details of the Step 4 of the WiSCR algorithm, we have that
\begin{itemize}[wide, nosep, labelindent = 15pt, topsep = 2ex]
    \item $a_1$ is the answer of $\mathcal{P}$, if only $a_1$ provides the 'most natural resolution' (By Definition \ref{def:nat_res}) for $p$ in $\mathcal{G_S}$,
    
    \item $a_2$ is the answer of $\mathcal{P}$, if only $a_2$ provides the `most natural resolution' for $p$ in $\mathcal{G_S}$,
    
    \item No answer otherwise%. %if none of the above two conditions are satisfied
\end{itemize}

% Then step 4 uses the isomorphism which is generated in Step 3 and outputs the answer to the WSC problem by using the following if then conditions,

% if .... then a1 is the answer

% if .... then a2 is the answer
\noindent The theorem is proved.
\end{proof}

\subsection{Proof of Theorem \ref{theorem:2}}
\begin{customthm}{2}
% 	Let $\mathcal{S}$ be a sequence of sentences in a WSC problem $\mathcal{P}$, $\mathbb{T}(S)$ be the set of tokens in $\mathcal{S}$, $p\in \mathbb{T}(S)$ be the token which represents the pronoun to be resolved, $a_1,a_2\in \mathbb{T}(S)$ be two tokens which represent the two answer choices,
%     % $\mathcal{P}$ be a WSC problem
%     $\mathcal{G_S}=(\mathbb{V}_\mathcal{S},\mathbb{E}_\mathcal{S},f_\mathcal{S})$ be a graphical representation of $\mathcal{S}$, %$\mathcal{G_Q}$ be a representation of the question in $\mathcal{P}$ (By Definition \ref{def:ques}), 
%     and $\mathcal{G_K}=(\mathbb{V}_\mathcal{K},\mathbb{E}_\mathcal{K},f_\mathcal{K})$ be a representation of a knowledge such that $f_\mathcal{K}$ is defined using $f_\mathcal{S}$. Also, $\Pi(\mathcal{G}_S,\mathcal{G}_K, p, a_1, a_2)$ be the AnsProlog program for WiSCR algorithm. Then, the WiSCR algorithm produces an answer $x$ to the input WSC problem iff:
	
% \begin{center}
% 	$\Pi(\mathcal{G}_S,\mathcal{G}_K, p, a_1, a_2) \vDash ans(x)$
% \end{center}
	Let $\mathcal{S}$ be a sequence of sentences in a WSC problem $\mathcal{P}$, $\mathbb{T}(S)$ be the set of tokens in $\mathcal{S}$, $p\in \mathbb{T}(S)$ be the token which represents the pronoun to be resolved, $a_1,a_2\in \mathbb{T}(S)$ be two tokens which represent the two answer choices,
    % $\mathcal{P}$ be a WSC problem
    $\mathcal{G_S}=(\mathbb{V}_\mathcal{S},\mathbb{E}_\mathcal{S},f_\mathcal{S})$ be a graphical representation of $\mathcal{S}$, %$\mathcal{G_Q}$ be a representation of the question in $\mathcal{P}$ (By Definition \ref{def:ques}), 
    and $\mathcal{G_K}=(\mathbb{V}_\mathcal{K},\mathbb{E}_\mathcal{K},f_\mathcal{K})$ be a representation of a knowledge such that $f_\mathcal{K}$ is defined using $f_\mathcal{S}$. Also, $\Pi(\mathcal{G}_S,\mathcal{G}_K, p, a_1, a_2)$ be the AnsProlog program for WiSCR algorithm and \textsc{AnswerFinder} be the python procedure defined in Section 4.2.5. Then, the WiSCR algorithm produces an answer $x$ to the input WSC problem iff 	$\Pi(\mathcal{G}_S,\mathcal{G}_K, p, a_1, a_2)$ and \textsc{AnswerFinder} together output the answer $x$.
\end{customthm}

\begin{proof}
(i) Given the ASP encoding of a graphical representation of the sequence of sentences in a WSC problem, the rules \texttt{s11}-\texttt{s13} extract a subgraph such that it contains only the non class nodes from the original graphs and the edges which connect them. The nodes of the subgraph are represented using the predicate \texttt{node\_G\_s} and the edges are represented using the binary predicate \texttt{edge\_G\_S}. In other words, the rules \texttt{s11}-\texttt{s13} implement the Step 1 of the WiSCR algorithm.

\noindent
(ii) Similar to (i) the rules \texttt{s21}-\texttt{s23} implement the Step 2 of the WiSCR algorithm.

\noindent
(iii) Given the outputs of the rules \texttt{s11}-\texttt{s23}, and the ASP representations of the sequence of sentences in a WSC problem and a knowledge, the rules \texttt{s31}-\texttt{s37} first generate all possible matching pairs corresponding to the nodes of the graph of WSC sentences and the graph of knowledge, then a set of constraints are used to remove the possibilities which do not represent an isomorphism between the subgraphs of WSC sentences and knowledge. In other words, the rules \texttt{s31}-\texttt{s37} implement the Step 3 of the WiSCR algorithm.

\noindent
(iv) Given an output of the rules \texttt{s31}-\texttt{s37}, and the ASP representations of the sequence of sentences in a WSC problem and a knowledge, the rules \texttt{s41}-\texttt{s49}  

\begin{itemize}[wide, nosep, labelindent = 0pt, topsep = 0ex]
    \item output $ans(a_1)$ if $matches(p,n_1)$, $matches(a_1,n_2)$ are true and $has\_k(n_1,"is\_same\_a\-s",n_2)$ or $has\_k(n_2,``is\_same\_as'',n_1)$ is true, and there does not exist an $n$ and an $x$ ($x\neq a_1$) such that $matches(x,n)$ is true and either $has\_k(n_1,"is\_same\_as",n)$ or $has\_k(n,"is\_same\_as",n_1)$ is true.
    
    \item output $ans(a_2)$ if $matches(p,n_1)$, $matches(a_2,n_2)$ are true and $has\_k(n_1,"is\_same\_a\-s",n_2)$ or $has\_k(n_2,``is\_same\_as'',n_1)$ is true, and there does not exist an $n$ and an $x$ ($x\neq a_2$) such that $matches(x,n)$ is true and either $has\_k(n_1,"is\_same\_as",n)$ or $has\_k(n,"is\_same\_as",n_1)$ is true.

    \item do not satisfy the current interpretation 
\end{itemize}

If more than one answers are produced and all of them correspond to one answer then \textsc{AnswerFinder} module outputs that as the final answer. Otherwise if zero answers are produced, or not all among the multiple answers correspond to a common answer then the \textsc{AnswerFinder} module does not output anything.

In other words, the rules \texttt{s41}-\texttt{s49} and the \textsc{AnswerFinder} module together implement the step 4 of the WiSCR algorithm.

By (i), (ii), (iii) and (iv), the WiSCR algorithm produces an answer $x$ to the input WSC problem iff 	$\Pi(\mathcal{G}_S,\mathcal{G}_K, p, a_1, a_2)$ and \textsc{AnswerFinder} together output the answer $x$.

\noindent The theorem is proved.

\end{proof}

\label{lastpage}
\end{document}